\documentclass[submission,copyright,creativecommons]{eptcs}
\providecommand{\event}{TARK-2017} 
\usepackage{underscore}            

\usepackage{bm}
\usepackage{amssymb}
\usepackage{amsmath}
\usepackage{amsthm}
\usepackage{enumitem}
\usepackage{tikz}
\newtheorem{theorem}{Theorem}
\newtheorem{definition}{Definition}
\newtheorem{example}{Example}

\newtheorem{proposition}[theorem]{Proposition}

\title{Preservation of Semantic Properties during the Aggregation of Abstract Argumentation Frameworks}
\author{Weiwei Chen
\institute{Institute of Logic and Cognition and Department of Philosophy\\
Sun Yat-sen University, P.R.\ China}
\institute{ILLC, University of Amsterdam, The Netherlands}
\email{chenww26@mail2.sysu.edu.cn}
\and
Ulle Endriss
\institute{Institute for Logic, Language and Computation\\
University of Amsterdam, The Netherlands}
\email{ulle.endriss@uva.nl}}
\def\titlerunning{Aggregation of Abstract Argumentation Frameworks}
\def\authorrunning{Weiwei Chen and Ulle Endriss}

\newcommand{\myendofexample}{~\hfill$\vartriangle$}
\newcommand{\myendofproof}{~\hfill$\Box$}
\renewcommand{\geq}{\geqslant}
\renewcommand{\leq}{\leqslant}
\renewcommand{\phi}{\varphi}

\newcommand{\prof}[1]{{\bm{#1}}}

\newcommand{\AF}{\text{\it AF}}   
\newcommand{\Arg}{\text{\it Arg}}          
\newcommand{\attacks}{\rightharpoonup}     
\newcommand{\att}{\text{\it att}}  
\newcommand{\Att}{\text{\it Att}}  

\newenvironment{AFfour}[1][]%
{\begin{tikzpicture}[->,>=latex,thick,shorten >=1pt,font=\small,scale=#1]
\node[draw,circle] (A) at (0,1.5) {$A$};
\node[draw,circle] (B) at (1.5,1.5) {$B$};
\node[draw,circle] (C) at (1.5,0) {$C$};
\node[draw,circle] (D) at (0,0) {$D$};}%
{\end{tikzpicture}}
\newcommand{\drawattack}[2]{\draw[->] (#1) edge (#2);}

\newcommand{\drawlabelledattack}[4]{\draw[->,dotted] (#1) edge node[#4] {{\footnotesize #3}} (#2);}

\begin{document}
\maketitle

\begin{abstract}
An abstract argumentation framework can be used to model the argumentative stance of an agent at a high level of abstraction, by indicating for every pair of arguments that is being considered in a debate whether the first attacks the second. When modelling a group of agents engaged in a debate, we may wish to aggregate their individual argumentation frameworks to obtain a single such framework that reflects the consensus of the group. Even when agents disagree on many details, there may well be high-level agreement on important semantic properties, such as the acceptability of a given argument. Using techniques from social choice theory, we analyse under what circumstances such semantic properties agreed upon by the individual agents can be preserved under aggregation.   
\end{abstract}

\section{Introduction}
Formal argumentation theory provides tools for modelling both the arguments an agent may wish to employ in a debate and the relationships that hold between such arguments~\cite{BesnardHunter2008,ModgilEtAlTAFA2011,RahwanSimari2009}. In the widely used system of \emph{abstract argumentation}, introduced in the seminal work of Dung~\cite{DungAIJ1995}, we abstract away from the internal structure of arguments and only model whether or not one argument \emph{attacks} another argument. This is a useful perspective when we require a high-level understanding of how different arguments relate to each other.
%
But when several agents engage in a debate, they may differ on their assessment of some of the arguments and their relationships. How best to model such scenarios of \emph{collective argumentation} is a question of considerable interest. Over the past decade or so, several authors have started to contribute to its resolution (see, e.g.,~\cite{AiriauEtAlAAMAS2016,BodanzaEtAlAC2017,CaminadaPigozziJAAMAS2011,CosteMarquisEtAlAIJ2007,DelobelleEtAlIJCAI2015,DunneEtAlCOMMA2012,RahwanTohmeAAMAS2010,TohmeEtAlFoIKS2008}).

Specifically, when agents differ on their assessment of which attacks between the arguments are in fact justified, i.e., when they put forward different \emph{attack-relations}, we may wish to \emph{aggregate} these individual pieces of information to obtain a global view. 
In this paper, we analyse under what circumstances a given \emph{aggregation rule} will \emph{preserve} relevant properties of the individual attack-relations, particularly properties that relate to the various \emph{semantics} that have been proposed for abstract argumentation. 
For example, if all agents agree that argument~$A$ is \emph{acceptable}, either because it is not attacked by any other argument or because it can be successfully defended against any such attack, then we would like $A$ to also be considered acceptable relative to the attack-relation returned by our aggregation rule. Thus, \emph{argument acceptability} is an example for a property that, ideally, should be preserved under aggregation.

Our approach uses techniques originating in \emph{social choice theory}, the study of collective decision making~\cite{ArrowEtAlHBSCW2002,Gaertner2006}, the relevance of which to collective argumentation has previously been noted by several of the aforementioned authors, starting with Tohm\'e et al.~\cite{TohmeEtAlFoIKS2008}. 
In particular, we make use of recent results on \emph{graph aggregation}~\cite{EndrissGrandiAIJ2017}.
%
Besides the formulation of a clear and simple model for the axiomatic study of the preservation of semantic properties during aggregation,
our contribution consists in delineating how fundamental axiomatic properties of aggregation rules interact with such preservation requirements. Our technical results range from characterisation results that indicate what kind of aggregation rule can satisfy certain combinations of desiderata, to impossibility results that show that only aggregation rules that are clearly unacceptable from an axiomatic point of view (namely, so-called dictatorships) can preserve the most demanding semantic properties.

This paper is organised as follows. Section~\ref{sec:argumentation} is a brief review of relevant concepts from the theory of abstract argumentation. Section~\ref{sec:model} introduces our model and Section~\ref{sec:results} present our technical results on the preservation of semantic properties under aggregation. Finally, we discuss related work in Section~\ref{sec:related-work} and conclude, in Section~\ref{sec:conclusion}, with suggestions for possible directions for future work.

\section{Abstract Argumentation}\label{sec:argumentation}

In this section, we recall some of the fundamentals of abstract argumentation as originally introduced by Dung~\cite{DungAIJ1995}.
An \emph{argumentation framework} is a pair $\AF=\langle\Arg,\attacks\rangle$, where $\Arg$ is a finite set of \emph{arguments} and $\attacks$ is an irreflexive binary relation on~$\Arg$.\footnote{Neither the finiteness nor the irreflexivity assumption are crucial for our results, but they simplify exposition and clearly are natural for most applications.} If $A\attacks B$ holds for two arguments $A,B\in\Arg$, we say that $A$ \emph{attacks} $B$. An \emph{attack-relation} $(\attacks)\subseteq\Arg\times\Arg$ is made up of a set of \emph{attacks} $\att\in\Arg\times\Arg$. For a set of arguments $\Delta\subseteq\Arg$ and an argument $B\in\Arg$, we say that $\Delta$ attacks $B$, denoted as $\Delta\attacks B$, if $A\attacks B$ holds for some argument $A\in\Delta$. 
Given an argumentation framework~$\AF$, the question arises which arguments to accept. For example, we may not want to accept two arguments that attack each other. A \emph{semantics} for abstract argumentation specifies which sets of arguments can be accepted together. Any such set of arguments is called an \emph{extension} of $\AF$ under the semantics in question.

For all the definitions of specific choices of semantics that follow, 
consider an arbitrary but fixed argumentation framework $\AF=\langle\Arg,\attacks\rangle$
and a set of arguments $\Delta\subseteq\Arg$.
We say that $\Delta$ is \emph{conflict-free}, if there exist no arguments $A,B\in\Delta$ such that $A\attacks B$.
We further say that $\Delta$ \emph{defends} the argument $B\in\Arg$, if $\Delta\attacks A$ for all arguments $A\in\Arg$ such that $A\attacks B$. 
Finally, $\Delta$ is called \emph{admissible} if it is conflict-free and defends every single one of its members.

\begin{definition}
A \textbf{stable extension} of $\AF$ is a conflict-free set $\Delta$ of arguments in $\Arg$ that attacks every other argument $B\in\Arg\setminus\Delta$. 
\end{definition}

\begin{definition}
A \textbf{preferred extension} of $\AF$ is an admissible set of arguments in $\Arg$ that is maximal with respect to set inclusion.
\end{definition}

\begin{definition}
A \textbf{complete extension} of $\AF$  is an admissible set of arguments in $\Arg$ that includes all of the arguments it defends.
\end{definition}

The \emph{characteristic function} of $\AF$ is the function $f_\AF : 2^\Arg\to 2^\Arg$ with $f_\AF : \Delta \mapsto \{A\in\Arg \mid \Delta\ \text{defends}\ A\}$, mapping any given set of arguments in $\Arg$ to the set of arguments it defends.

\begin{definition}
The \textbf{grounded extension} of $\AF$ is the least fixed point of its characteristic function $f_\AF$.
\end{definition}

We can compute the grounded extension~$\Delta$ by initialising $\Delta$ with the empty set~$\emptyset$ and then repeatedly executing the program $\Delta := f_\AF(\Delta)$, until no more changes occur.
Unlike for the other three semantics, there always is exactly one grounded extension. However, that extension may be empty. It is nonempty if and only if there is at least one unattacked argument.

How do these semantics relate to each other?
Every stable extension is also a preferred extension.
The set of stable extensions may be empty, while there always is at least one preferred extension.
Every preferred extension is also a complete extension. 
Finally, the grounded extension is always a complete extension as well.

An interesting question is under what circumstances two or more 
semantics coincide. Probably the clearest example 
is the case of 
an \emph{acyclic} attack-relation: if $\attacks$ does not include any cycles, then the grounded extension is also the only stable extension, the only preferred extension, and the only complete extension. 
A weaker condition 
is \emph{coherence:} $\AF$ is called \emph{coherent} if every preferred extension of $\AF$ is stable, i.e., if the two semantics coincide.

\section{The Model}\label{sec:model}

Fix a finite set of arguments~$\Arg$. Let $N=\{1,\ldots,n\}$ be a finite set of $n$ \emph{agents}. Suppose each agent $i\in N$ supplies us with an argumentation framework $\AF_i=\langle\Arg,\attacks_i\rangle$, reflecting her individual views on the status of possible attacks between arguments. Thus, we are given a \emph{profile} of attack-relations $\prof{\attacks}=(\attacks_1,\ldots,\attacks_n)$.\footnote{Note that we assume that all agents report an attack-relation over \emph{the same} set of arguments~$\Arg$. As argued by Coste-Marquis et al.~\cite{CosteMarquisEtAlAIJ2007}, generalisations, where different agents may be aware of different subsets of $\Arg$, are possible and interesting, but---in line with most existing work in the area---we shall not explore them here.} What would be a good method of aggregating these individual argumentation frameworks to arrive at a single argumentation framework that appropriately reflects the views of the group as a whole? This is the central question we address in this paper. An \emph{aggregation rule} is a function $F : (\Arg\times\Arg)^n\to\Arg\times\Arg$ mapping any given profile of attack-relations into a single attack-relation.

\begin{example}
The first aggregation rule that comes to mind is the majority rule: include attack $A\attacks B$ in the outcome if and only if a (weak) majority of the individual agents do. If we apply this rule to the profile shown in Figure~\ref{fig:firstexample}, we obtain an argumentation framework with $A\attacks B$, $B\attacks C$, and $C\attacks A$.\myendofexample
\end{example}

\begin{figure}[t]
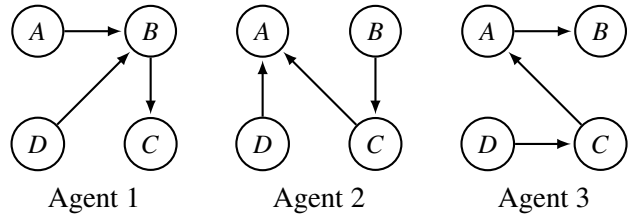

\[
\begin{tabular}{c@{\qquad}c@{\qquad}c}
\begin{AFfour}[1] 
\drawattack{A}{B}\drawattack{B}{C}\drawattack{D}{B}
\end{AFfour} &
\begin{AFfour}[1]
\drawattack{B}{C}\drawattack{C}{A}\drawattack{D}{A}
\end{AFfour} &
\begin{AFfour}[1]
\drawattack{C}{A}\drawattack{A}{B}\drawattack{D}{C}
\end{AFfour} \\
Agent 1 & Agent 2 & Agent 3
\end{tabular}
\]
\vspace*{-5pt}
\caption{Example for a profile with $\Arg=\{A,B,C,D\}$.\label{fig:firstexample}}
\end{figure}

In the remainder of this section, we first define a number of specific aggregation rules and review their properties. 
We focus on simple rules that are adaptations of well-known rules used in the social choice literature, particularly in judgment aggregation~\cite{GrossiPigozzi2014} and graph aggregation~\cite{EndrissGrandiAIJ2017}. We then adapt several standard \emph{properties of aggregation rules}, known as \emph{axioms} in that same literature, to our setting. Finally, we review several semantic \emph{properties of argumentation frameworks}
and formulate the question of whether a given rule will \emph{preserve} such a property. 

\subsection{Aggregation Rules}\label{sec:rules}

Recall that an aggregation rule is a function $F$, mapping any given profile $\prof{\attacks}=(\attacks_1,\ldots,\attacks_n) \in (\Arg\times\Arg)^n$ of attack-relations 
to a single attack-relation $F({\prof{\attacks}}) \subseteq \Arg\times\Arg$. We sometimes write $(A\attacks B) \in F(\prof{\attacks})$ for $(A,B)\in F(\prof{\attacks})$. We use $N^{\prof{\attacks}}_{\att} := \{i\in N \mid \att \in (\attacks_i)\}$ to denote the set of \emph{supporters} of the attack $\att$ in profile~$\prof{\attacks}$.

\begin{definition}
Let $q\in\{1,\ldots,n\}$. The \textbf{quota rule}~$F_q$ with quota~$q$ accepts all those attacks that are supported by at least~$q$ agents:
\begin{eqnarray*}
F_q(\prof{\attacks}) & = & \left\{ \att\in\Arg\times\Arg \mid \#N^\prof{\attacks}_{\att} \geq q \right\}
\end{eqnarray*}
\end{definition}

The \emph{weak majority rule} is the quota rule $F_q$ with $q=\lfloor\frac{n}{2}\rfloor$ and the \emph{strict majority rule} is the quota rule $F_q$ with $q=\lceil\frac{n}{2}\rceil$. 
Two further quota rules are also of special interest. 
The \emph{unanimity rule} only accepts attacks that are supported by everyone, i.e., this is $F_q$ with $q=n$. The \emph{nomination rule} is the quota rule $F_q$ with $q=1$. Despite being a somewhat extreme choice, the nomination rule has some intuitive appeal in the context of argumentation, as it reflects the idea that we should take seriously any conflict between arguments raised by at least one member of the group.

\begin{definition}
Let $C\in 2^N\setminus\{\emptyset\}$ be a nonempty coalition of agents. 
The \textbf{oligarchic rule}~$F_C$ accepts all those attacks that are accepted by all members of~$C$:
\begin{eqnarray*}
F_C(\prof{\attacks}) & = & \left\{ \att\in\Arg\times\Arg \mid C\subseteq N^\prof{\attacks}_{\att} \right\}
\end{eqnarray*}
\end{definition}

Thus, any member of the oligarchy~$C$ can \emph{veto} an attack from being accepted.
Observe that the unanimity rule can also be characterised as the oligarchic rule~$F_C$ with $C=N$.
A subclass of the oligarchic rules are the \emph{dictatorships}. The dictatorship of dictator~$i\in N$ is the oligarchic rule~$F_C$ with $C=\{i\}$. Thus, under a dictatorship, to compute the outcome, we simply copy the attack-relation of the dictator. Intuitively speaking, oligarchic rules, and dictatorships in particular, are unattractive rules, as they unfairly exclude everyone not in $C$ from the decision process. 

Some rules combine features of the quota rules and the oligarchic rules. For example, we may choose to accept an attack only if it is accepted by $(i)$~a weak majority of all agents and $(ii)$~a small number of distinguished agents to which we want to give the right to veto 
attacks. Such rules (sometimes called \emph{qualified majority rules}) are certainly more attractive than the oligarchic rules, but they are still unfair in the sense of granting some agents more influence than others.

\begin{definition}
Agent~$i\in N$ has \textbf{veto powers} under aggregation rule~$F$, if $F(\prof{\attacks}) \subseteq (\attacks_i)$ for every profile~$\prof{\attacks}$.
\end{definition}

Thus, under an oligarchic rule~$F_C$ the agents in~$C$, and only those, have veto powers. With the exception of the unanimity rule, 
a quota rule does not grant veto powers to any agent. 

\subsection{Axioms}\label{sec:axioms}

Next, we introduce some basic \emph{axioms}, encoding intuitively desirable properties of an aggregation rule~$F$:

\begin{definition}
$F$ is 
\textbf{anonymous}, if $F(\prof{\attacks}) = F(\attacks_{\pi(1)},\ldots,\attacks_{\pi(n)})$ holds for all profiles $\prof{\attacks}$ and all permutations~$\pi:N\to N$.
\end{definition}


\begin{definition}
$F$ is 
\textbf{neutral}, if $N^{\prof{\attacks}}_{\att} = N^{\prof{\attacks}}_{\att'}$ implies $\att \in F(\prof{\attacks}) \Leftrightarrow \att' \in F(\prof{\attacks})$ for all profiles $\prof{\attacks}$ and all attacks $\att$, $\att'$.
\end{definition}


\begin{definition}
$F$ is 
\textbf{independent}, if $N^{\prof{\attacks}}_{\att} = N^{\prof{\attacks'}}_{\att}$ implies $\att \in F(\prof{\attacks}) \Leftrightarrow \att \in F(\prof{\attacks'})$ for all profiles $\prof{\attacks}$, $\prof{\attacks'}$ and all attacks $\att$.
\end{definition}


\begin{definition}
$F$ is 
\textbf{monotonic}, if $N^{\prof{\attacks}}_{\att} \subseteq N^{\prof{\attacks'}}_{\att}$ 
(together with $N^{\prof{\attacks}}_{\att'} = N^{\prof{\attacks'}}_{\att'}$ for all attacks $\att' \not= \att$)  
implies $\att \in F(\prof{\attacks}) \Rightarrow \att \in F(\prof{\attacks'})$ 
for all profiles $\prof{\attacks}$, $\prof{\attacks'}$ and all attacks $\att$.
\end{definition}


\begin{definition}
$F$ is 
\textbf{unanimous}, if $F(\prof{\attacks}) \supseteq (\attacks_1) \cap \cdots \cap (\attacks_n)$
holds for all profiles ${\prof{\attacks}}=(\attacks_1,\ldots,\attacks_n)$.
\end{definition}


\begin{definition}
$F$ is 
\textbf{grounded}, if $F(\prof{\attacks}) \subseteq (\attacks_1) \cup \cdots \cup (\attacks_n)$
holds for all profiles ${\prof{\attacks}}=(\attacks_1,\ldots,\attacks_n)$.
\end{definition}

Anonymity is a symmetry (and thus fairness) requirement regarding agents, and neutrality is a symmetry requirement regarding attacks. Independence expresses that whether an attack is accepted should only depend on its supporters. Monotonicity says that additional support for an accepted attack should never cause it to be rejected. Unanimity postulates that an attack supported by everyone must be accepted, while groundedness means that only attacks with at least one supporter can be collectively accepted.\footnote{Note that, in line with the existing literature in argumentation theory on the one hand and social choice theory on the other, we use the term ``grounded'' in two unrelated ways in this paper (grounded extensions \textit{vs.}\ grounded aggregation rules).}

We observe that all quota rules and all oligarchic rules are easily seen to be unanimous, grounded, neutral, independent, and monotonic. The quota rules furthermore are also anonymous. In fact, it is not difficult to adapt a well-known result from judgment aggregation to our setting~\cite{DietrichListJTP2007}, so as to see that the quota rules are \emph{the only} aggregation rules that satisfy all of these six axioms (see also~\cite{EndrissGrandiAIJ2017}).

\subsection{Preservation of Semantic Properties}

Typically, agents will disagree about which attacks 
are justified.
But even then, there may be high-level agreement on certain features. For example, they may all agree that, under a particular semantics, argument~$A$ is acceptable. 
Whenever we observe such agreement on semantic features in a profile, we would like those features to be preserved under aggregation. 
Thus, for our example, under the same semantics, we would like $A$ to be acceptable also in the argumentation framework computed by our aggregation rule.
%
In other words, we are interested in the preservation of properties of argumentation frameworks 
under aggregation.
Formally, an \emph{AF-property}~$P \subseteq 2^{\Arg \times \Arg}$ is simply the set of all attack-relations on $\Arg$ that satisfy~$P$. But in the interest of readability, we write $P(\attacks)$ rather than $(\attacks)\in P$.

\begin{definition}
Let $F$ be an aggregation rule and let $P$ be an AF-property.
We say that $F$ \textbf{preserves} $P$, if for every profile~$\prof{\attacks}$ we have that $P(\attacks_i)$ being the case for all agents $i\in N$ implies $P(F(\prof{\attacks}))$.
\end{definition}

This notion of preservation is known under the name of \emph{collective rationality} in other parts of social choice theory~\cite{Arrow1963,EndrissGrandiAIJ2017,ListPettitEP2002}.

We now review the specific AF-properties for which we study preservation in this paper.
Two of them 
we have already introduced in Section~\ref{sec:argumentation}, namely \emph{acyclicity} and \emph{coherence}. They are attractive properties, because---if satisfied by an argumentation framework---they reduce our dependence on a specific choice of semantics, thereby making decisions less controversial.

Recall that the grounded extension may be empty, i.e., this semantics may not suggest any arguments to be accepted. Thus, argumentation frameworks that satisfy the AF-property of \emph{nonemptiness of the grounded extension} are of particular interest.

Let $A\in\Arg$ be one of the arguments under consideration. Then, for any given argumentation framework, $A$ may or may not belong to the grounded extension. Thus, every $A\in\Arg$ defines an AF-property, namely the property of membership of~$A$ in the grounded extension, i.e., of acceptance of $A$ under the grounded semantics. 
We say that $F$ \emph{preserves argument acceptability under the grounded semantics}, if it is the case that, for all arguments $A\in\Arg$, whenever $A$ belongs to the grounded extension of $\langle\Arg,\attacks_i\rangle$ for all 
$i\in N$, then $A$ also belongs to the grounded extension of $\langle\Arg,F(\prof{\attacks})\rangle$.
For the stable, preferred, and complete semantics, we require a more refined definition, given that extensions under these semantics need not be unique. 
We say that $F$ \emph{preserves argument acceptability under the stable semantics}, if it is the case that, for all arguments $A\in\Arg$, whenever $A$ belongs to \emph{some} stable extension of $\langle\Arg,\attacks_i\rangle$ for all agents $i\in N$, then $A$ also belongs to \emph{some} stable extension of $\langle\Arg,F(\prof{\attacks})\rangle$.
The corresponding concepts for the preferred and the complete semantics are defined accordingly.\footnote{A further distinction between acceptability in \emph{some} extension and acceptability in \emph{all} extensions would be possible. We leave the investigation of this issue to future work.}

Rather than just preserving the acceptability status of a single argument, we may also be interested in preserving entire extensions. For example, we say that $F$ \emph{preserves extensions under the stable semantics}, if it is the case that, for all sets $\Delta\subseteq\Arg$, whenever $\Delta$ is a stable extension of $\langle\Arg,\attacks_i\rangle$ for all agents $i\in N$, then $\Delta$ is also a stable extension of $\langle\Arg,F(\prof{\attacks})\rangle$. Again, the corresponding concept can also be defined for the other three semantics. Similarly, we say that $F$ \emph{preserves conflict-freeness}, if it is the case that, for all sets $\Delta\subseteq\Arg$, whenever $\Delta$ is conflict-free in $\langle\Arg,\attacks_i\rangle$ for all agents $i\in N$, then $\Delta$ is also conflict-free $\langle\Arg,F(\prof{\attacks})\rangle$. Finally, \emph{preservation of admissibility} is defined accordingly.

To summarise, we have identified the following AF-properties that, in case all agents agree on one of them being satisfied, we would like to see preserved under aggregation:
\begin{itemize}[itemsep=1pt]
\item \emph{acyclicity} and \emph{coherence} (reducing semantic ambiguity),
\item \emph{nonemptiness of the grounded extension} (enabling a sceptical approach to argument evaluation),
\item \emph{argument acceptability} under different semantics (allowing for agreement on arguments even in the face of disagreement on the attacks between them), and
\item \emph{being an extension} under one of the four semantics or being either a \emph{conflict-free} or an \emph{admissible set} (also allowing for semantic agreement despite disagreement on attacks).
\end{itemize}

\begin{example}
Consider again the profile of Figure~\ref{fig:firstexample} and recall that the (weak or strict) majority rule will return the argumentation framework with $A\attacks B$, $B\attacks C$, and $C\attacks A$. Thus, the majority rule does not preserve acyclicity.\footnote{This observation is closely related to the famous \emph{Condorcet Paradox} in the theory of preference aggregation \cite{McLeanUrken1995}.}
What about some of the other AF-properties?
The grounded extension of $\AF_1$ is $\{A,C,D\}$,
that of $\AF_2$ is $\{B,D\}$, 
that of $\AF_3$ is $\{A,D\}$, and
that of the majority outcome is $\{D\}$.
Thus, preservation of both nonemptiness of the grounded extension and argument acceptability under the grounded semantics are not violated by this particular example (which, of course, is not to say that they might not be violated for other profiles).\myendofexample   
\end{example}

\section{Preservation Results}\label{sec:results}

In this section, we present our results on the preservation of semantic properties under aggregation. We begin with argument acceptability, and then turn to the various properties of sets of arguments, 
and eventually to acyclicity and coherence. 
%
%


\subsection{Acceptability of an Argument}

Our first result is going to demonstrate that preserving acceptability of an argument when using a ``simple'' aggregation rule is impossible, unless we are willing to use a dictatorship. This is true under any of the four semantics. To prove this result---and some of those that follow---we are going to use a technique developed by Endriss and Grandi~\cite{EndrissGrandiAIJ2017} for the more general framework of graph aggregation, which in turn has been inspired by the seminal work on preference aggregation of Arrow~\cite{Arrow1963}. 
It amounts to showing that, under certain assumptions, the collection of coalitions of agents that are sufficiently powerful to force collective acceptance of an attack must form an ultrafilter.

Using our present terminology, Endriss and Grandi~\cite[Theorem~18]{EndrissGrandiAIJ2017} show that, for $|\Arg|\geq 3$, any unanimous, grounded, neutral, and independent aggregation rule~$F$ that preserves some AF-property~$P$ must be a dictatorship whenever $P$ belongs to what they call the family of \emph{implicative} and \emph{disjunctive} properties. 
$P$ is \emph{implicative} if there exist a set $\Att\subseteq\Arg\times\Arg$ of attacks and three attacks $\att_1,\att_2,\att_3\in\Arg\times\Arg\setminus\Att$ such that, for any $S\subseteq\{\att_1,\att_2,\att_3\}$, we have $P(\Att\cup S)$ if and only if $S\not=\{\att_1,\att_2\}$. Thus, $P$ should require that, in the context of $\Att$, accepting $\att_1$ and $\att_2$ implies accepting $\att_3$ (and all seven patterns of acceptance consistent with that requirement are possible).
$P$ is \emph{disjunctive} if there exist $\Att\subseteq\Arg\times\Arg$ and $\att_1,\att_2\in\Arg\times\Arg\setminus\Att$ such that, for any $S\subseteq\{\att_1,\att_2\}$, we have $P(\Att\cup S)$ if and only if $S\not=\emptyset$. Thus, $P$ should require that, given $\Att$, we must accept at least one of $\att_1$ and $\att_2$ (and all three patterns of acceptance consistent with that requirement are possible).\footnote{Our definitions of implicativeness and disjunctiveness are special cases of the more general definitions given by Endriss and Grandi~\cite{EndrissGrandiAIJ2017}. They simplify exposition and are sufficient for our purposes here.}

In the appendix, we show that, under any of the four semantics and for $|\Arg|\geq 4$, the property of argument acceptability is both implicative and disjunctive. Thus, we obtain the following impossibility result, showing that there exists no aggregation rule that satisfies all of our requirements:

\begin{theorem}\label{thm:acceptability-dictator}
Let $P$ be the property of argument acceptability under either the grounded, the stable, the preferred, or the complete semantics. For $|\Arg|\geq 4$, any unanimous, grounded, neutral, and independent aggregation rule~$F$ that preserves $P$ must be a dictatorship.
\end{theorem}

This is bad news. 
For scenarios with $|\Arg|\leq 3$, we can do better and show that also the nomination rule preserves argument acceptability. In the interest of space, we omit the details and also do not discuss such boundary conditions for our remaining results.

\subsection{Conflict-Freeness and Admissibility}

Argument acceptability is a property that relates to a \emph{single} argument. Next, we turn to properties that relate to \emph{sets} of arguments. 
Our most basic property of sets of arguments is preserved under essentially all reasonable aggregation rules:

\begin{theorem}\label{thm:grounded-conflictfree}
Every aggregation rule~$F$ that is grounded preserves conflict-freeness.
\end{theorem}

\begin{proof}
Let $F$ be an aggregation rule that is \emph{grounded}. 
%
Consider any set $\Delta\subseteq\Arg$ and any profile $\prof{\attacks}=(\attacks_1,\ldots,\attacks_n)$ such that $\Delta$ is conflict-free in $\langle\Arg,\attacks_i\rangle$ for all $i\in N$. 
For the sake of contradiction, assume $\Delta$ is \emph{not} conflict-free in $\langle\Arg,F(\prof{\attacks})\rangle$, i.e., there exist two arguments $A,B\in\Delta$ such that $(A\attacks B)\in F(\prof{\attacks})$. Due to the groundedness of $F$, there then must be at least one agent $i\in N$ such that also $A\attacks_i B$, i.e., $\Delta$ is not conflict-free in $\langle\Arg,\attacks_i\rangle$ either, in contradiction to our original assumption.
\end{proof}

For admissibility, we obtain a significantly less broad but still positive result. It shows that there exists a reasonable rule that preserves the admissibility of arbitrary sets of arguments: 

\begin{theorem}\label{thm:admissibility-nomination}
For $|\Arg|\geq 4$, the only unanimous, grounded, anonymous, neutral, independent, and monotonic aggregation rule~$F$ that preserves admissibility is the nomination rule.
\end{theorem}

\begin{proof}
We first show that the \emph{nomination rule} indeed preserves \emph{admissibility}.
So let $F$ be the nomination rule.
Consider any set $\Delta\subseteq\Arg$ and any profile $\prof{\attacks}=(\attacks_1,\ldots,\attacks_n)$ such that $\Delta$ is admissible in $\AF_i = \langle\Arg,\attacks_i\rangle$ for all $i\in N$.
For the sake of contradiction, assume $\Delta$ is \emph{not} admissible in $\langle\Arg,F(\prof{\attacks})\rangle$, i.e., there is an argument $A \in \Delta$ that, in $F(\prof{\attacks})$, is attacked by an argument $B \in \Arg \setminus \Delta$ and there does not exist a $C \in \Delta$ such that $(C\attacks B)\in F(\prof{\attacks})$.
As $(B \attacks A) \in F(\prof{\attacks})$ and as $F$ is grounded, we must have $B \attacks_i A$ for some $i \in N$. 
And as there does not exist a $C \in \Delta$ such that $(C \attacks A) \in F(\prof{\attacks})$, given the definition of the nomination rule, there cannot exist an argument $C \in \Delta$ such that $C \attacks_i A$ for that same agent~$i$. Hence, $\Delta$ is not admissible in $\AF_i$, in contradiction to our original assumption.

We still need to show that there can be no other aggregation rule than the nomination rule that preserves admissibility and that satisfies all of the axioms mentioned in the statement of Theorem~\ref{thm:admissibility-nomination}. By the characterisation result for quota rules due to Dietrich and List~\cite{DietrichListJTP2007} in the context of judgment aggregation, which has been adapted to graph aggregation by Endriss and Grandi~\cite{EndrissGrandiAIJ2017} and which we have briefly recalled near the end of Section~\ref{sec:rules}, this claim is equivalent to the claim that \emph{no quota rule} $F_q$ with a quota $q>1$ always preserves admissibility. So let us prove this.

\begin{figure}[t]
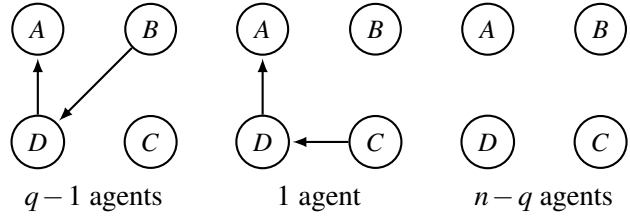

\[
\begin{tabular}{c@{\qquad}c@{\qquad}c}
\begin{AFfour}[1]  
\drawattack{D}{A}\drawattack{B}{D}
\end{AFfour} &
\begin{AFfour}[1] 
\drawattack{D}{A}\drawattack{C}{D}
\end{AFfour} &
\begin{AFfour}[1] 
\end{AFfour} \\
$q\,{-}\,1$ agents & $1$ agent & $n\,{-}\,q$ agents
\end{tabular}
\]
\vspace*{-5pt}
\caption{Profile used in the proof of Theorem~\ref{thm:admissibility-nomination}.\label{fig:admissibility-nomination-proof}}
\end{figure}

Consider the generic profile shown in Figure~\ref{fig:admissibility-nomination-proof} (and note that $q>1$ ensures $q-1>0$, i.e., there is at least one agent of the first kind). The set $\{A,B,C\}$ is admissible in all argumentation frameworks in such a profile. But when we aggregate using a quota rule $F_q$ with a quota $q>1$, we obtain an argumentation framework with a single attack $D\attacks A$, which means that $A$ cannot be part of any admissible set. Hence, no such rule can preserve admissibility.
\end{proof}

\subsection{The Property of Being an Extension}

We continue our examination of properties relating to sets of arguments and turn to the property of being an extension of a given argumentation framework. The following result can again be proved by reference to the result of Endriss and Grandi~\cite{EndrissGrandiAIJ2017}:

\begin{theorem}\label{thm:grounded-extension-dictator}
For $|\Arg|\geq 5$, any unanimous, grounded, neutral, and independent aggregation rule~$F$ that preserves grounded extensions must be a dictatorship.
\end{theorem} 

The required proof of the fact that all three properties are both implicative and disjunctive can be found in the appendix.
We conjecture that Theorem~\ref{thm:grounded-extension-dictator} can be extended to the preferred and the complete semantics. The added difficulty lies in the fact that these semantics admit multiple extensions.
Interestingly, for 
stable extensions we obtain a much more positive result:

\begin{proposition}\label{prop:nomination-stable}
The nomination rule preserves stable extensions.
\end{proposition}

\begin{proof}
Let $F$ be the \emph{nomination rule}. 
Consider any set $\Delta\subseteq\Arg$ and any profile $\prof{\attacks}=(\attacks_1,\ldots,\attacks_n)$ such that $\Delta$ is stable in $\langle\Arg,\attacks_i\rangle$ for all $i\in N$.
According to Theorem~\ref{thm:grounded-conflictfree}, given that $F$ is grounded, $F$ preserves conflict-freeness. Thus, $\Delta$ is conflict-free in $\langle \Arg, F(\prof{\attacks}) \rangle$. 

What remains to be shown is that $\Delta$ attacks every argument $B\in\Arg\setminus\Delta$.
In case $\Delta = \Arg$, the claim holds vacuously. 
Otherwise, consider an arbitrary argument $B\in\Arg\setminus\Delta$. We need to show that $B$ is attacked by some argument in $\Delta$ in $F(\prof{\attacks})$. 
Take the argumentation framework $\AF_i = \langle\Arg,\attacks_i\rangle$ for some $i \in N$. As $\Delta$ is stable in $\AF_i$ by assumption, there exists an argument $A\in\Delta$ such that $A\attacks_i B$. As $F$ is the nomination rule, we also get $(A \attacks B) \in F(\prof{\attacks})$ as claimed. 
\end{proof}

\subsection{Nonemptiness of the Grounded Extension}

We have seen that preserving grounded extensions is impossible for reasonable aggregation rules
(see Theorem~\ref{thm:grounded-extension-dictator}). 
What about the seemingly less demanding requirement of at least preserving nonemptiness of the grounded extension? The good news is that we can do better for this property. For instance, it is easy to check that the unanimity rule preserves nonemptiness of the grounded extension. Still, as we shall see next, we cannot do \emph{much} better: only rules that grant veto powers to some agents will work.

Recall that the grounded extension is nonempty if an only if at least one argument is not attacked by any other argument. Thus, this AF-property is about the \emph{absence} of attacks, while the technique we employed to prove Theorem~\ref{thm:grounded-extension-dictator} (and Theorem~\ref{thm:acceptability-dictator}) exploits the \emph{presence} of certain attacks. 
We are now going to present our preservation result regarding the nonemptiness of the grounded extension as a corollary to a more general theorem about the preservation of AF-properties that require the absence of certain attacks. 
Let $k\in\mathbb{N}$. Let us call an AF-property $P$ \emph{$k$-exclusive} if there exist $k$ distinct attacks $\att_1,\ldots,\att_k\in\Arg\times\Arg$ such that
$(i)$~$\{\att_1,\ldots,\att_k\}\subseteq(\attacks)$ for no attack-relation $\attacks$ with $P(\attacks)$, and
$(ii)$~for every $S\subsetneq\{\att_1,\ldots,\att_k\}$ there exists an attack-relation $\attacks$ such that $S\subseteq(\attacks)$ and $P(\attacks)$.
Thus, you cannot accept all $k$ attacks, but you should be able to accept any proper subset of them.
We are able to prove the following powerful theorem (recall that $n$ is the number of agents in~$N$):


\begin{theorem}\label{thm:k-exclusive-veto}
Let $k\geq n$ and let $P$ be an AF-property that is $k$-exclusive.
Then under any neutral and independent aggregation rule~$F$ that preserves~$P$ at least one agent must have veto powers.
\end{theorem}

\begin{proof}
Let $k\geq n$, let $P$ be an AF-property that is \emph{$k$-exclusive}, and let $F$ be an aggregation rule that is \emph{neutral} and \emph{independent}. We need to show that, if $F$ preserves $P$, then $F$ must give some agents the power to veto the collective acceptance of attacks.

First, observe that, if an aggregation rule~$F$ is neutral and independent, then we can describe $F$ by listing all the coalitions $C\subseteq N$ for which it is the case that, whenever exactly the agents in $C$ support an attack~$\att$, then $\att$ must be collectively accepted. Indeed, independence says that acceptance of an attack should only depend on its supporters, and neutrality adds that this dependence must be the same for all attacks. More formally, there exists a family of \emph{winning coalitions} $\mathcal{W}\subseteq 2^N$ such that, for all profiles $\prof{\attacks}$ and all potential attacks $\att\in\Arg\times\Arg$, the following relationship holds: 
\begin{center}
$\att\in F(\prof{\attacks})$ 
\;\; if and only if \;\;
$N^{\prof{\attacks}}_\att\in\mathcal{W}$
\end{center}

Recall that $i\in N$ having veto powers under $F$ means that $F(\prof{\attacks})\subseteq(\attacks_i)$ for every profile~$\prof{\attacks}$. We can now show that an agent~$i\in N$ has veto powers, if she is a member of all winning coalitions:
\begin{center}
$\displaystyle i\in\bigcap_{C\in\mathcal{W}} C$
\;\; implies \;\;
$F(\prof {\attacks}) \subseteq (\attacks_i)$ for all profiles $\prof{\attacks}$ 
\end{center}

If $\bigcap_{C\in\mathcal{W}} C = \emptyset$, then the above claim holds vacuously.
Otherwise, take any attack $\att\in F(\prof{\attacks})$. As $\att$ got accepted, $N^\prof{\attacks}_\att$ must be a winning coalition, i.e., $N^{\prof{\attacks}}_\att \in \mathcal{W}$ and therefore $i\in N^{\prof{\attacks}}_\att$. But this is just another way of saying $\att \in (\attacks_i)$, so we are done.

Next, we are going to show that the fact that $F$ preserves the $k$-exclusive AF-property $P$ implies that the intersection of any $k$ winning coalitions must be nonempty:
\begin{center}
$C_1\cap\cdots\cap C_k\not=\emptyset$ for all $C_1,\ldots,C_k\in\mathcal{W}$
\end{center} 

For the sake of contradiction, assume there do exist winning coalitions $C_1,\ldots,C_k\in\mathcal{W}$ such that $C_1 \cap \cdots \cap C_k = \emptyset$. We construct a profile ${\prof{\attacks}}=(\attacks_1,\ldots,\attacks_n)$ with $P(\attacks_i)$ for all $i\in N$ as follows: for every $j\in\{1,\ldots,k\}$, exactly the agents in $C_j$ accept attack~$\att_j$ (for all other attacks, it is irrelevant which agents accept them). As no agent is a member of all $k$ winning coalitions, no agent accepts all $k$ attacks, so this construction indeed is possible. However, as each of the $k$ attacks is supported by a winning coalition, they all get accepted, i.e., $\{\att_1,\ldots,\att_k\} \subseteq F(\prof{\attacks})$, meaning that the outcome does \emph{not} satisfy~$P$. Thus, we have found a contradiction to our assumption of $F$ preserving~$P$ and are done.

Let us briefly recap where we are at this point.
We know that $F$ is characterised by a family of winning coalitions $\mathcal{W}$.
We also know that $C_1\cap\cdots\cap C_k\not=\emptyset$ for all $C_1,\ldots,C_k\in\mathcal{W}$. 
We need to show that some agents have veto powers, and we know that this is the case if we can prove that $C^{(1)}\cap\cdots\cap C^{(\ell)} \not= \emptyset$, where $\{C^{(1)},\ldots,C^{(\ell)}\}$ is some enumeration of the coalitions in $\mathcal{W}$.
Thus, we are done, if we can show that $C_1\cap\cdots\cap C_k\not=\emptyset$ for \emph{all} $C_1,\ldots,C_k\in\mathcal{W}$ implies $C^{(1)}\cap\cdots\cap C^{(\ell)} \not= \emptyset$.
We are going to prove the contrapositive, namely that the following holds for \emph{some} $C_1,\ldots,C_k\in\mathcal{W}$:
\begin{center}
$C^{(1)}\cap\cdots\cap C^{(\ell)} = \emptyset$
\;\; implies \;\;
$C_1\cap\cdots\cap C_k = \emptyset$ 
\end{center}

In words, we need to show that in case the intersection of \emph{all} winning coalitions is empty, then so is at least one intersection of just $k$ winning coalitions.

Recall that we have assumed $k\geq n$.
We construct a set $\mathcal{W}'\subseteq\mathcal{W}$ of $k$ (or fewer) winning coalitions as follows.
Initially, set $\mathcal{W}' := \emptyset$.
Then, for every $j$ from $1$ to $\ell$ in turn, add $C^{(j)}$ to $\mathcal{W}'$ if and only if the following condition is satisfied:\footnote{By convention, let $\bigcap_{C\in\emptyset} C = N$, i.e., the intersection of \emph{no} winning coalitions is defined as the universe~$N$ of \emph{all} agents.} 
\begin{eqnarray*}
\left( C^{(j)} \;\cap \bigcap_{C\in\mathcal{W}'} \!\! C \right) 
& \subsetneq & 
\left( \bigcap_{C\in\mathcal{W}'} \!\! C \right)
\end{eqnarray*}

Thus, every additional $C^{(j)}$ is selected only if it causes the removal of at least one further agent from the intersection. As there are only $n$ agents, we therefore will pick at most $n$ coalitions. Hence, we will indeed arrive at a family $\mathcal{W}'$ of $n$ or fewer---and thus certainly at most $k$---winning coalitions, the intersection of which is empty. This completes the proof.
\end{proof}

It now suffices to show that the property of having a nonempty grounded extension is an $|\Arg|$-exclusive property to obtain the following result (the proof of this fact can be found in the appendix):

\begin{theorem}\label{thm:nonempty-veto}
If $|\Arg|\geq n$, then under any neutral and independent aggregation rule~$F$ that preserves nonemptiness of the grounded extension at least one agent must have veto powers.
\end{theorem}

It is not difficult to prove that the converse holds as well: all rules that grant veto powers to at least one agent preserve nonemptiness of the grounded extension. 

\subsection{Acyclicity and Coherence}


Finally, if we apply our techniques to the properties of acyclicity and coherence, we obtain the following results (the proofs of which can be found in the appendix):\footnote{Theorem~\ref{thm:acyclic-veto} was anticipated in the work of Tohm\'e et al.~\cite{TohmeEtAlFoIKS2008}, who make a similar claim, 
but without appealing to the neutrality axiom. 
We stress that Theorem~\ref{thm:acyclic-veto} cannot be strengthened by dropping neutrality. Indeed, there are rules that preserve acyclicity, that are independent (but not neutral), and that do not give veto powers to any of the agents. An example, for $N=\{1,2\}$ and $\Arg=\{A,B\}$, is the rule that accepts $A\attacks B$ if at least one agent does and that accepts $B\attacks A$ if both agents do.}

\begin{theorem}\label{thm:acyclic-veto}
If $|\Arg|\geq n$, then under any neutral and independent aggregation rule~$F$ that preserves acyclicity at least one agent must have veto powers.
\end{theorem}




\begin{theorem}\label{thm:coherence-dictator}
For $|\Arg|\geq 4$, any unanimous, grounded, neutral, and independent aggregation rule~$F$ that preserves coherence must be a dictatorship.
\end{theorem}

Thus, surprisingly, 
even though acyclicity is a stronger property than coherence, it is easier to preserve under aggregation.

\section{Related Work}\label{sec:related-work}

In this section, we briefly review related work on the aggregation of abstract argumentation frameworks.
Coste-Marquis et al.~\cite{CosteMarquisEtAlAIJ2007} were the first to address this problem, but without making explicit reference to social choice theory. Instead, they focus on 
aggregation rules that minimise the \emph{distance} between the input argumentation frameworks and the output argumentation framework.

Tohm\'e et al.~\cite{TohmeEtAlFoIKS2008} were the first to explicitly use social choice theory to analyse the aggregation of argumentation frameworks. Their focus is on the preservation of acyclicity and they show that qualified majority rules (which grant veto powers to some agents) will always preserve this property. 


Dunne et al.~\cite{DunneEtAlCOMMA2012} define several preservation requirements on aggregation rules that directly refer to the semantics of the argumentation frameworks concerned. This includes variants of what we call preservation of extensions (``$\sigma$-unanimity'') and preservation of argument acceptability (``credulous acceptance unanimity''). Their focus is on analysing the computational complexity of deciding whether a given aggregation rule has a given property, rather than on the axiomatic method.
In follow-up work, Delobelle et al.~\cite{DelobelleEtAlIJCAI2015} establish for several concrete rules whether or not they satisfy the preservation requirements introduced by Dunne et al.~\cite{DunneEtAlCOMMA2012}.

While Endriss and Grandi~\cite{EndrissGrandiAIJ2017} explicitly mention abstract argumentation as a possible domain of application for the model of graph aggregation they develop, they do not present any technical results related to argumentation.

Airiau et al.~\cite{AiriauEtAlAAMAS2016} introduce the concept of the \emph{rationalisability} of a profile of argumentation frameworks. A profile is rationalisable if the diversity of views it contains can be explained in terms of $(i)$~an underlying factual argumentation framework shared by all agents and $(ii)$~everyone's individual preferences. Thus, their work is concerned with understanding what kind of profiles a good aggregation rule should be able to deal with, rather than with aggregation itself.

Finally, social choice theory has also been used to analyse the aggregation of different \emph{extensions} for \emph{one} argumentation framework see, (e.g.,~\cite{CaminadaPigozziJAAMAS2011,RahwanTohmeAAMAS2010}). We note that this problem is different from the one studied here and refer to the survey by Bodanza et al.~\cite{BodanzaEtAlAC2017} for a comparison. 

\section{Conclusion}\label{sec:conclusion}

Using a variety of techniques, we have attempted to paint a clear picture of the capabilities and limitations of simple aggregation rules regarding the preservation of properties 
related to the semantics of abstract argumentation frameworks. While the significance of this issue and the promise of social choice theory for its resolution have previously been emphasised in the work of several authors~\cite{BodanzaEtAlAC2017,DelobelleEtAlIJCAI2015,DunneEtAlCOMMA2012,TohmeEtAlFoIKS2008}, this is the first systematic analysis of its kind. Our results show that only the most basic of properties, namely conflict-freeness, is preserved by essentially all rules. More demanding properties require either the nomination rule, a rule granting some agents veto powers, or a rule that is dictatorial. 

We stress that these results only apply to simple rules, in particular, to rules that satisfy the axiom of independence. An alternative route, the one chosen by Coste-Marquis et al.~\cite{CosteMarquisEtAlAIJ2007}, is to use distance-based rules (which violate independence). Such rules can be designed so as to guarantee specific properties of the outcome, so the question of preservation does not arise. On the downside, distance-based rules are computationally intractable~\cite{EndrissEtAlJAIR2012,HemaspaandraEtAlTCS2005,KoniecznyEtAlAIJ2004}.
We also stress that our results are based on the assumption that all agents report attack-relations over a single \emph{common} set of arguments. Richer models, where different agents may be aware of different sets of arguments, are clearly of great interest as well.

There are multiple directions in which to extend this work. 
First, our conjecture regarding the preservation of preferred and complete extensions should get settled.
Second, one could study further properties of argumentation frameworks. We have already hinted at the possibility of distinguishing between argument acceptability in \emph{some} extensions (the property studied in this paper) and argument acceptability in \emph{all} extensions. 
Third, we should eventually go beyond the four classical semantics introduced by Dung~\cite{DungAIJ1995} and also consider others, such as the \emph{semi-stable semantics}~\cite{CaminadaCOMMA2006} or the \emph{ideal semantics}~\cite{DungEtAlAIJ2007}.
Fourth, one could vary the axioms imposed on aggregation rules. The most immediately promising direction here would be to investigate whether neutrality can be replaced by additional preservation requirements, in analogy to results in preference and graph aggregation~\cite{Arrow1963,EndrissGrandiAIJ2017}. 
Finally, it would be interesting to investigate the strategic incentives of agents who are reporting an argumentation framework to an aggregation rule and whose objective might be to get a certain argument accepted. 






\paragraph{Acknowledgements.}
We would like to thank Sirin Botan, Umberto Grandi, Ronald de Haan, and Zoi Terzopoulou for numerous enlightening discussions on the material presented in this paper and three anonymous reviewers for their constructive feedback.

\bibliographystyle{eptcsini}
\bibliography{aaaf}

\appendix
\section*{Appendix: Remaining Proofs}

In this appendix, we present the proofs omitted from the body of the paper. All of the proofs in this appendix have the same structure: they show that a given semantic AF-property of interest has certain meta-properties, for which a general (impossibility) result is available.

\subsection*{Proof of Theorem~\ref{thm:acceptability-dictator}}

Recall that, for each of the four semantics, we have to show that, for $|\Arg|\geq 4$, the AF-property of \emph{argument acceptability} is both \emph{implicative} and \emph{disjunctive}. We are able to use the same construction in all four cases.
Let $P$ be the AF-property of argument acceptability under either the grounded, the stable, the preferred, or the complete semantics. Furthermore, let $\Arg = \{A,B,C,D,\ldots\}$ be a set of at least four arguments.

\begin{figure}[t]
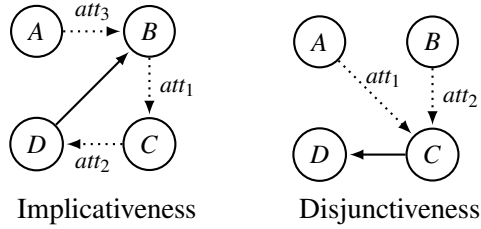

\[
\begin{tabular}{c@{\qquad\quad}c}
\begin{AFfour}[1] 
\drawattack{D}{B}
\drawlabelledattack{B}{C}{$\att_1$}{right}
\drawlabelledattack{C}{D}{$\att_2$}{below}
\drawlabelledattack{A}{B}{$\att_3$}{above}
\end{AFfour} &
\begin{AFfour}[1] 
\drawattack{C}{D}
\drawlabelledattack{A}{C}{$\;\;\;\att_1$}{above}
\drawlabelledattack{B}{C}{$\att_2$}{right}
\end{AFfour} \\
Implicativeness & Disjunctiveness
\end{tabular}
\]
\vspace*{-5pt}
\caption{Scenarios used in the proof of Theorem~\ref{thm:acceptability-dictator}.\label{fig:proof-acceptability-dictator}}
\end{figure}

Let us first show that $P$ is implicative. 
Let $C$ be the argument under consideration.
Let $\Att = \{D \attacks B\}$, $\att_1 = B \attacks C$, $\att_2 = C \attacks D$, and $\att_3 = A \attacks B$.
This scenario is sketched in the lefthand part of Figure~\ref{fig:proof-acceptability-dictator}.
Now consider the argumentation frameworks of the form $\langle\Arg,\Att\cup S\rangle$ with $S\subseteq\{\att_1,\att_2,\att_3\}$.
If $S\subseteq\{\att_2,\att_3\}$, then $C$ is not attacked by any other argument.
If $S=\{\att_1\}$ or $S=\{\att_1,\att_3\}$, then $C$ is defended by $D$, which is not attacked by any other argument.
If $S=\{\att_1,\att_2,\att_3\}$, then $C$ is defended by $A$, which is not attacked by any other argument.
Thus, in all of these seven cases, either $C$ is not attacked by any other argument or it is defended by an argument that is not attacked by any other argument, i.e., in all cases $C$ is acceptable under the grounded, the stable, the preferred, and the complete semantics.
On the other hand, if $S=\{\att_1,\att_2\}$, then $\{B, C, D\}$ forms an isolated odd-length cycle. This means that all of $B$, $C$, and $D$ will be unacceptable under the grounded, the stable, the preferred, and the complete semantics.
We have thus found a set of attacks $\Att$ and three individual attacks $\att_1$, $\att_2$, $\att_3$ such that $P(\Att \cup S)$ if and only if $S\not=\{\att_1,\att_2\}$. Hence, $P$ is an implicative AF-property.

Next, we show that $P$ is also disjunctive. 
Let $D$ be the argument under consideration.
Let $\Att = \{C \attacks D\}$, $\att_1 = A \attacks C$, and $\att_2 = B \attacks C$.
This scenario is depicted on the righthand side of Figure~\ref{fig:proof-acceptability-dictator}.
Consider all argumentation frameworks $\langle\Arg,\Att\cup S\rangle$ with $S\subseteq\{\att_1,\att_2\}$.
If $S=\{\att_1\}$, then $D$ is defended by $A$.
If $S=\{\att_2\}$, then $D$ is defended by $B$.
If $S=\{\att_1,\att_2\}$, then $D$ is defended by both $A$ and $B$.
In all three cases, $D$ is defended by some argument that is not attacked by any other argument. Thus, $D$ is acceptable under each of our four semantics.
However, if $S=\emptyset$, then $D$ is attacked by $C$ and not defended by any other argument, which means that $D$ is unacceptable under all four semantics.
To summarise, we have seen that $P(\Att \cup S)$ if and only if $S \not= \emptyset$. Hence, $P$ is a disjunctive AF-property.
\myendofproof



\subsection*{Proof of Theorem~\ref{thm:grounded-extension-dictator}}

Recall that we need to show that, for $|\Arg|\geq 5$, the property~$P$ of \emph{being a grounded extension} is an AF-property that is both \emph{implicative} and \emph{disjunctive}. 
In fact, as we shall see, we can show disjunctiveness even in case $|\Arg|\geq 4$.
Let $\Arg = \{A,B,C,D,E,\ldots\}$.
For the proofs of both properties, we focus on $\Delta = \{A,C,E\}$ as the subset of arguments that may (or may not) form the grounded extension. 

We first show that $P$ is implicative.
Let $\Att = \{A \attacks B, C \attacks D\}$,
$\att_1 = B \attacks C$, $\att_2 = D \attacks A$, and $\att_3 = E \attacks D$.
This scenario is depicted on the lefthand side of Figure~\ref{fig:proof-grounded-extension-dictator}.
Consider all argumentation frameworks of the form $\AF=\langle\Arg,\Att\cup S\rangle$ with $S\subseteq\{\att_1,\att_2,\att_3\}$, and for each of them the corresponding characteristic function $f_\AF$.
If $S=\{\att_1\}$, then $f_\AF(\emptyset)=\{A,E\}$, $f_\AF^2(\emptyset)=\{A,C,E\}=\Delta$, and $f_\AF^3(\emptyset)=f_\AF^2(\emptyset)$. Thus, the grounded extension is $\Delta$ in this case.
Using the same kind of reasoning, it is easy to verify that the grounded extension is $\Delta$ whenever $S\not=\{\att,\att_2\}$.
On the other hand, if $S=\{\att_1,\att_2\}$, then we get $f_\AF(\emptyset)=\{E\}$ and $f_\AF^2(\emptyset)=f_\AF(\emptyset)$, i.e., the grounded extension now is $\{E\}$.
Thus, for our argumentation framework to have $\Delta$ as its grounded extension, we must insist that, if both $\att_1$ and $\att_2$ are accepted, then also $\att_3$ is accepted. Hence, $P$ is implicative.

\begin{figure}[t]
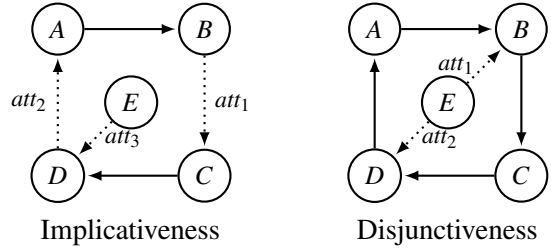

\[
\begin{tabular}{c@{\qquad\quad}c}
\begin{AFfour}[1.3]  
\node[draw,circle] (E) at (0.75,0.75) {$E$};
\drawattack{A}{B}
\drawattack{C}{D}
\drawlabelledattack{B}{C}{$\att_1$}{right}
\drawlabelledattack{D}{A}{$\att_2$}{left}
\drawlabelledattack{E}{D}{$\att_3$}{right}
\end{AFfour} &
\begin{AFfour}[1.3]
\node[draw,circle] (E) at (0.75,0.75) {$E$};
\drawattack{A}{B}
\drawattack{B}{C}
\drawattack{C}{D}
\drawattack{D}{A}
\drawlabelledattack{E}{B}{$\att_1$}{left}
\drawlabelledattack{E}{D}{$\att_2$}{right}
\end{AFfour} \\
Implicativeness & Disjunctiveness
\end{tabular}
\]
\vspace*{-5pt}
\caption{Scenarios used in the proof of Theorem~\ref{thm:grounded-extension-dictator}.\label{fig:proof-grounded-extension-dictator}}
\end{figure}

Next, we show that $P$ is also a disjunctive AF-property.
Let $\Att = \{A \attacks B,B \attacks C,C \attacks D,D \attacks A\}$,
$\att_1 = E \attacks B$, and $\att_2 = E \attacks D$.
This is shown on the righthand side of Figure~\ref{fig:proof-grounded-extension-dictator}.
Consider $\AF = \langle \Arg, \Att\cup S \rangle$ with $S\subseteq\{\att_1,\att_2\}$.
If $S=\{\att_1\}$, then $f_\AF(\emptyset)=\{E\}$, $f_\AF^2(\emptyset)=\{C,E\}$, $f_\AF^3(\emptyset)=\{A,C,E)=\Delta$, and $f_\AF^4(\emptyset)=f_\AF^3(\emptyset)$, i.e., $\Delta$ is the grounded extension of~$\AF$. By analogous reasoning, $\Delta$ is the grounded extension also for $S=\{\att_2\}$ and for $S=\{\att_1,\att_2\}$. However, for $S=\emptyset$, we get $f_\AF(\emptyset)=\{E\}$ and $f_\AF^2(\emptyset)=f_\AF(\emptyset)$, i.e., the grounded extension now is just $\{E\}$. Thus, $\Delta$ is the grounded extension if and only if $S\not=\emptyset$, meaning that $\att_1$ or $\att_2$ are accepted. Hence, $P$ is disjunctive.
\myendofproof



\subsection*{Proof of Theorem~\ref{thm:nonempty-veto}}

To obtain the claim as a corollary to Theorem~\ref{thm:k-exclusive-veto}, we need to show that the property of an argumentation framework \emph{having a nonempty grounded extension} is a \emph{$k$-exclusive} AF-property for $k=|\Arg|$. 
Recall that having a nonempty grounded extension is equivalent to the property of \emph{having at least one argument that is not attacked} by any other argument. We are going to show that the latter property is $k$-exclusive for $k=|\Arg|$.

So let $k=|\Arg|$. Take an arbitrary enumeration $\{A^{(1)},\ldots,A^{(k)}\}$ of $\Arg$ and consider the set of attacks $\{\att_1,\ldots,\att_k\}$ with $\att_i := A^{(i)} \attacks A^{(i+1)}$ for $i<k$ and $\att_k := A^{(k)} \attacks A^{(1)}$. Clearly, this set of attacks meets our requirements:
$(i)$~if $\{\att_1,\ldots,\att_k\}\subseteq(\attacks)$, then $\attacks$ does not have the property of leaving at least one argument without an attacker and
$(ii)$~for every $S\subsetneq\{\att_1,\ldots,\att_k\}$ there exists an attack-relation $\attacks$ with $S\subseteq(\attacks)$, namely $S$ itself, that does leave one argument without an attacker.
\myendofproof

\subsection*{Proof of Theorem~\ref{thm:acyclic-veto}}

Recall that, if we can show that \emph{acyclicity} is an $|\Arg|$-exclusive AF-property, then the claim follows from Theorem~\ref{thm:k-exclusive-veto}. In fact, it is straightforward to show that acyclicity is a \emph{$k$-exclusive} property for every $k\in\{2,\ldots,|\Arg|\}$.
To see this, consider the case where the attacks $\{\att_1,\ldots,\att_k\}$ form a cycle, and observe that the shortest (proper) cycle has length~2, while the longest cycle visits every argument exactly once and thus has length~$|\Arg|$.
\myendofproof

\subsection*{Proof of Theorem~\ref{thm:coherence-dictator}}

Recall that we need to show that, for $|Arg|\geq 4$, \emph{coherence} is an AF-property that is both \emph{implicative} and \emph{disjunctive}. Let $P$ represent coherence and let $\Arg = \{A,B,C,D,\ldots\}$.

\begin{figure}[t]
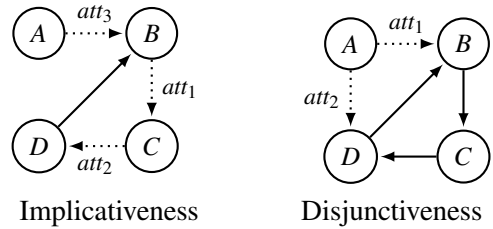

\[
\begin{tabular}{c@{\qquad\quad}c}
\begin{AFfour}[1] 
\drawattack{D}{B}
\drawlabelledattack{B}{C}{$\att_1$}{right}
\drawlabelledattack{C}{D}{$\att_2$}{below}
\drawlabelledattack{A}{B}{$\att_3$}{above}
\end{AFfour} &
\begin{AFfour}[1] 
\drawattack{B}{C}
\drawattack{C}{D}
\drawattack{D}{B}
\drawlabelledattack{A}{B}{$\att_1$}{above}
\drawlabelledattack{A}{D}{$\att_2$}{left}
\end{AFfour} \\
Implicativeness & Disjunctiveness
\end{tabular}
\]
\vspace*{-5pt}
\caption{Scenarios used in the proof of Theorem~\ref{thm:coherence-dictator}.\label{fig:proof-coherence-dictator}}
\end{figure}

Let us first show that $P$ is an implicative AF-property. 
Let $\Att = \{D \attacks B\}$, $\att_1 = B \attacks C$, $\att_2 = C \attacks D$, and $\att_3 = A \attacks B$.
This scenario is shown on the lefthand side of Figure~\ref{fig:proof-coherence-dictator} and is identical to the scenario used in the proof of Theorem~\ref{thm:acceptability-dictator}.
Now consider argumentation frameworks $\langle\Arg,\Att\cup S\rangle$ with $S\subseteq\{\att_1,\att_2,\att_3\}$.
If either $S=\{\att_1\}$, $S=\{\att_3\}$, $S=\{\att_1,\att_3\}$ or $S=\emptyset$, the only preferred extension is $\{A,C,D\}$, which is also stable.
If $S=\{\att_2\}$, the only preferred extension is $\{A,B,C\}$, which is also stable. 
If $S=\{\att_2,\att_3\}$ or $S=\{\att_1,\att_2,\att_3\}$, the only preferred extension is $\{A,C\}$ which once again also is stable.
Thus, in all seven cases, $\AF$ is coherent.
However, if $S=\{\att_1,\att_2\}$, the only preferred extension is $\{A\}$, which is not stable, as $B$, $C$, and $D$ are not attacked by $A$.
We have thus found a set of attacks $\Att$ and three individual attacks $\att_1$, $\att_2$, $\att_3$ such that $P(\Att \cup S)$ if and only if $S\not=\{\att_1,\att_2\}$. In other words, $P$ is an implicative AF-property.

Next, we show that $P$ is also a disjunctive AF-property. 
Let $\Att = \{B \attacks C, C \attacks D, D \attacks B\}$, $\att_1 = A \attacks B$, and $\att_2 = A \attacks D$.
This scenario is shown on the righthand side of Figure~\ref{fig:proof-coherence-dictator}.
Consider argumentation frameworks $\langle\Arg,\Att\cup S\rangle$ with $S\subseteq\{\att_1,\att_2\}$.
If $S=\{\att_1\}$ or $S=\{\att_1,\att_2\}$, the only preferred extension is $\{A,C\}$, which is also stable.
If $S=\{\att_2\}$, the only preferred extension is $\{A,B\}$, which again is also stable.
Thus, in all three cases every preferred extension is stable, i.e.,  $\AF$ is coherent.
On the other hand, if $S=\emptyset$, then the only preferred extension is $\{A\}$, which is not stable, as $B$, $C$, and $D$ not attacked by $A$.
To summarise, $P(\Att \cup S)$ if and only if $S \not= \emptyset$. Hence, $P$ is disjunctive.
\myendofproof

\end{document}